%% file: main.tex
\documentclass[letterpaper, 10pt, conference]{ieeeconf}

\input{preamble}

\IEEEoverridecommandlockouts  
\usepackage[english]{babel}
\usepackage{footnotehyper}

\usepackage{hyperref,tablefootnote,footnotehyper}
\usepackage{amsmath,booktabs,mwe,pdflscape}
\hypersetup{colorlinks}

\usepackage{leftindex} 
\usepackage{graphicx} 
\usepackage{amsmath}
\usepackage{amssymb}
\usepackage{amsthm}
\usepackage{amsfonts}
\usepackage{float}
\usepackage{mathrsfs}
\usepackage{booktabs}
\usepackage{array}
\usepackage[nodisplayskipstretch]{setspace}
\makeatletter
\def\BState{\State\hskip-\ALG@thistlm}
\makeatother
\usepackage[capitalize,nameinlink,noabbrev]{cleveref} 
\theoremstyle{plain}

\theoremstyle{plain}
\newtheorem{prob}{Problem}
\theoremstyle{definition}
\newtheorem{assumption}{Assumption}
\newtheorem{remark}{Remark}

\newcolumntype{M}[1]{>{\centering\arraybackslash}m{#1}}
\usepackage{pifont}
\usepackage{multirow}
\usepackage{tabularray}
\usepackage{cite}
\usepackage{algpseudocode}

\newcommand{\RNum}[1]{\uppercase\expandafter{\romannumeral #1\relax}}
\usepackage{bigints}

\newcommand{\gatekeeper}{\texttt{gatekeeper}}

\algrenewcommand\textproc{}

\usepackage[font=small,skip=10pt]{caption}

\algrenewcommand\algorithmicrequire{\textbf{Input:}}
\algrenewcommand\algorithmicensure{\textbf{Output:}}

\hypersetup{
  colorlinks,
  citecolor=teal, 
  linkcolor=teal,
  urlcolor=purple}

\author{Kaleb Ben Naveed$^{1}$, Devansh R. Agrawal$^{1}$, and Dimitra Panagou$^{1,2}$%
\thanks{$^{*}$The authors would like to acknowledge the support of the National Science Foundation (NSF) under grant no. 2223845 and grant no. 1942907.}
\thanks{$^{1}$Department of Robotics, University of Michigan, Ann Arbor, MI, 48109 USA. 
{\tt\small \{kbnaveed@umich.edu\}}}
\thanks{$^{2}$Department of Aerospace Engineering, University of Michigan, Ann Arbor, MI, 48109 USA. }}

\title{\LARGE \bf
A Formal \gatekeeper{} Framework for Safe Dual Control with Active Exploration
} 

\newcommand\footnoteref[1]{\protected@xdef\@thefnmark{\ref{#1}}\@footnotemark}

\begin{document}
\maketitle
\thispagestyle{empty}
\pagestyle{empty}

\begin{abstract}
\input{subsections/abstract}
\end{abstract}


\section{Introduction}
Planning safe trajectories under model uncertainty is a fundamental challenge in control and robotics \cite{naveed2025enabling, lopez2019dynamic_robust2, lew2023risk_robust_3, sasfi2023robust}. Such uncertainty often arises from unknown parameters, such as aerodynamic drag or tire–road friction. Robust frameworks—including tube-based MPC \cite{lopez2019dynamic_robust2, agrawal2024gatekeeper}, sampling-based methods \cite{blackmore2010probabilistic, lew2023risk_robust_3}, control barrier function based methods \cite{9482751, 9683085}, and contraction approaches \cite{7989693, sasfi2023robust}—guarantee constraint satisfaction under worst-case disturbances and model mismatch. Yet, these guarantees come at the cost of conservatism, since the methods do not actively reduce model uncertainty during execution.

A complementary perspective is offered by the dual control problem, which explicitly balances exploitation (achieving the control objective) with exploration (actively reducing uncertainty) \cite{feldbaum1961theory}. Here, uncertainty can be divided into aleatoric, reflecting irreducible randomness, and epistemic, stemming from unknown but learnable parameters. Dual control focuses on reducing the latter while pursuing the mission objective. Many approaches have been proposed in this direction \cite{lew2022safe, parsi2020active, Kim-RSS-23, davydov2025first, sasfi2023robust_robust_1, Barcelos-RSS-21, hu2024active, hibbard2023safely, prajapat2025safe, luo2024act, parsi2023dual, mesbah2018stochastic, li2022dual, arcari2020dual, knaup2024adaptive, johnson2025implicit, soloperto2020augmenting, Zhang-RSS-25}. Some focus on decoupling the two phases, first reducing uncertainty and then planning robust policies \cite{lew2022safe, parsi2020active, Kim-RSS-23, davydov2025first}. Others rely on passive reduction, using data naturally gathered during the mission to reduce model uncertainty \cite{sasfi2023robust_robust_1, Barcelos-RSS-21, hu2024active, lopez2019adaptive}. More recent work leverages active exploration to reduce model uncertainty in parallel with mission execution. \cite{hibbard2023safely, prajapat2025safe, luo2024act, parsi2023dual}.

Despite progress, several challenges remain. Safety is not inherent in most dual control formulations, with exploration often pursued without guarantees of constraint satisfaction \cite{parsi2020active, Barcelos-RSS-21, luo2024act, arcari2020dual}. Also, the exploration–exploitation trade-off is usually managed by adding a weighted exploration term to the cost \cite{parsi2023dual, Barcelos-RSS-21}, without a principled decision on when exploration should occur. As a result, exploration may drive mission cost too high when the weight is too large, or be suppressed when too small, hindering uncertainty reduction and limiting performance gains.

To address these challenges, we propose a framework that integrates safety, active exploration, and budget feasibility \cite{cherenson2025autonomy} in trajectory planning under bounded model uncertainty. Specifically, the robot must respect state and input constraints, actively reduce parameter uncertainty through exploration, and ensure that the cost functional (which can express e.g., control effort, execution time, or tracking error) remains within a prescribed budget. Our approach is inspired by the \gatekeeper{} architecture \cite{agrawal2024gatekeeper}. At each planning cycle, a robust backup trajectory is first computed to guarantee safety and provide a feasible fallback. In parallel, informative candidate trajectories are generated to promote parameter identification. Each candidate is evaluated against the backup by predicting its impact on both uncertainty reduction and mission cost, and it is committed only if it reduces uncertainty while remaining within the budget; otherwise, the system defaults to the backup. This ensures that exploration is pursued only when provably safe, beneficial, and budget-feasible, thereby addressing the dual control trade-off while guaranteeing safety. We demonstrate the effectiveness of this framework on two case studies.
\subsubsection*{\textbf{Notation}}
Let $\mathbb{R}$, $\mathbb{R}_{\geq 0}$, and $\mathbb{R}_{> 0}$ denote the set of real, non-negative real, and positive real numbers, respectively. 

\section{Preliminaries \& Problem Formulation}\label{sec:prob}
\subsubsection{System Model}
\label{sec:system_model}
We consider a class of nonlinear control-affine systems with parametric uncertainty and bounded additive disturbance:
\begin{align}
    \dot{x} &= 
    f_{0}(x) + F(x)\theta_f + \big(g_{0}(x) + G(x)\theta_g
    \big)u+ w(t), \label{eq:system_model}
\end{align}
where $x \in \Xcal \subset \R^n$ is the state, $u \in \Ucal \subset \R^m$ is the control input, $\theta_f \in \Theta_f \subset \R^{p_f}$ is the unknown drift parameter vector contained in a known compact set, and $\theta_g = \begin{bmatrix}\theta_{g,1} & \cdots & \theta_{g,m}\end{bmatrix} \in \Theta_g \subset \R^{p_g \times m}$ is the unknown input parameter matrix contained in a known compact set, with $\theta_{g,j} \in \R^{p_g}$ denoting its $j$-th column. The uncertain parameters are collected into the vector as 
$\theta = \begin{bmatrix} \theta_f^\top & \theta_{g,1}^\top & \cdots & \theta_{g,m}^\top \end{bmatrix}^\top \in \Theta \subset \R^{p_f+mp_g}$.
The functions $f_0: \R^n \to \R^n$ and $g_0: \R^n \to \R^{n \times m}$ are the known nominal drift and input maps, respectively. The functions $F: \R^n \to \R^{n \times p_f}$ and $G: \R^n \to \R^{n \times p_g}$ are the known drift and input regressors, respectively. Furthermore, we denote $f(x,\theta_f) = f_{0}(x) + F(x)\theta_f$ and $g(x,u,\theta_g) = \big(g_{0}(x) + G(x)\theta_g
    \big)u$.
\begin{assumption}
\label{asm:noise_derivatives}
The additive disturbance $w:[t_0,\infty)\to\R^n$ is bounded $\sup_{t \ge t_0} \|w(t)\| = \overline{w}.$ 
\end{assumption}
\begin{assumption}
\label{asm:full_state} The full system state $x(t)$ $\forall t \in [t_0, \infty)$ is assumed to be perfectly observed.
\end{assumption}
\subsubsection{Linear-in-Parameter Regressor Form}
\label{sec:lip_form}
Under Assumption~\ref{asm:noise_derivatives} and Assumption~\ref{asm:full_state},
the system dynamics \eqref{eq:system_model} can be rearranged to obtain a regression model
that is linear in the unknown parameter vector~$\theta$.
Specifically, define the signal
\eqn{
z(t) = \dot{x}(t) - f_0(x(t)) - g_0(x(t))u(t).
}
\footnote{In practice, $\dot x(t)$ is reconstructed from sampled state measurements via finite differences. Any resulting bounded error can be absorbed in $w(t)$.}
Substituting \eqref{eq:system_model} yields
\eqn{
z(t) = F(x(t))\theta_f + (G(x(t))\theta_g) u(t) + w(t).
}
Using the representation
$\theta_g = \begin{bmatrix}\theta_{g,1} & \cdots & \theta_{g,m}\end{bmatrix}$,
the input-dependent term can be written as
\eqn{
(G(x)\theta_g) u
= \sum_{j=1}^m u_j G(x)\theta_{g,j}.
}
Define the regressor matrix
\eqn{
\Phi(x,u) =
\big[
F(x)\;\; u_1 G(x)\;\; \cdots\;\; u_m G(x)
\big]
\in \R^{n \times (p_f + m p_g)}.
}
Then the dynamics admit the linear-in-parameters (LIP) representation
\eqn{
\label{eq:lip_form_PE}
z(t) = \Phi(x(t),u(t))\,\theta + w(t).
}
The LIP model \eqref{eq:lip_form_PE} forms the basis for set-membership parameter
identification and uncertainty set refinement in the subsequent sections.
To ensure identifiability, the regressor must be persistently exciting (PE) \cite{PE_1_narendra1987persistent}, which is formally defined as follows:
\begin{definition}[Persistent Excitation]\label{def:PE}
Consider the LIP form in \eqref{eq:lip_form_PE} where $\Phi(x, u)$ is the regressor. The regressor is said to be persistently exciting if there exist constants $T >0$, $\alpha >0$, and $\beta >0$ such that $\forall t > t_0$, 
\begin{subequations}
  \eqn{
  \Gcal(t, T) = \int_{t}^{t+T} \Phi(x(\tau), &u(\tau))^\top \Phi(x(\tau), u(\tau))\, d\tau \\
\alpha I \leq \, &\Gcal(t, T) \, \leq \beta I,
}  
\end{subequations}
\end{definition}
The uncertainty in the parameter set can captured by a \textit{width} of the parameter set $\Theta$. 
\begin{definition}[Width of Parameter Set]
\label{def:width}
Let $\Theta \subset \R^p$ be a parameter set and $\Dcal$ be a finite set of unit directions.  
For any direction $d \in \Dcal$, the width of $\Theta$ along $d$ is defined as
\eqn{
w_d(\Theta) \;=\; \sup_{\theta \in \Theta} d^\top \theta \;-\; \inf_{\theta \in \Theta} d^\top \theta.
}
This quantity gives the extent of $\Theta$ along direction $d$. For $p=1$ with $d=1$ and $\Theta=[\theta_{\min},\theta_{\max}]$, it simplifies to $w_d(\Theta)=\theta_{\max}-\theta_{\min}$.
\end{definition}

\subsubsection{Problem Statement}
We provide definitions needed for the problem formulation. We consider a dual control setting where the mission must be completed while reducing uncertainty in bounded parameters, with the state restricted to a safe set $\Scal \subseteq \Xcal$ and the cost limited by a given budget.

\begin{definition}[Trajectory]
\label{def:traj}
Let $\Tcal=[t_i,t_f]\subset\R$. A trajectory is a pair $ p \;=\; \big(p_x:\Tcal\to\Xcal,\;\; p_u:\Tcal\to\Ucal\big)$
such that
\eqn{
\dot p_x(t)
= f\big(p_x(t),\hat\theta_f\big)
+ g\big(p_x(t),p_u(t),\hat\theta_g\big),
\quad \forall\, t\in\Tcal .
\label{eq:traj_dynamics}
}
Here, $\hat\theta_f\in\Theta_f$ and $\hat\theta_g\in\Theta_g$ are fixed parameter estimates used to define the nominal model.
\end{definition}

The pair $(p_x, p_u)$ thus specifies a nominal state–input trajectory for the estimated model. Building on this, we define the robust tube trajectory.

\begin{figure*}[t]
  \centering
  \includegraphics[width=2.05\columnwidth]{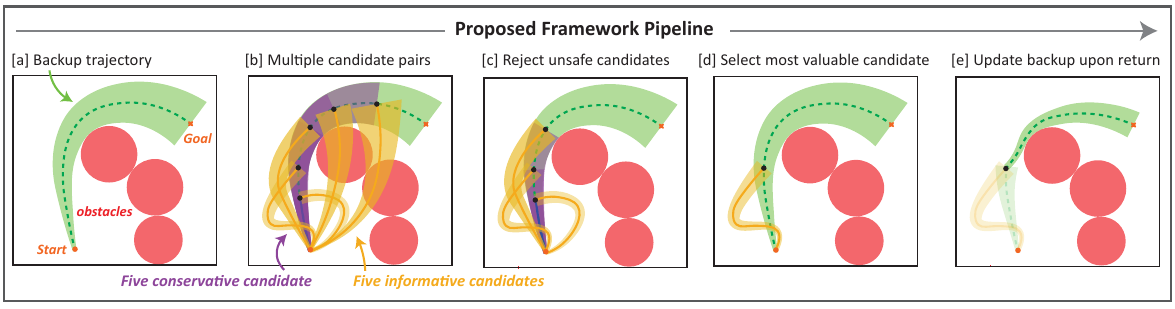}
  \caption{The proposed framework at a glance. Starting from a conservative backup trajectory, multiple candidate trajectories are generated, including both conservative and informative options. Unsafe candidates are rejected, and the most valuable safe candidate is selected for execution. Upon returning to the backup trajectory, the backup plan is updated. The candidate trajectories represent variations in the trajectory state space; their depiction in the physical environment is purely illustrative.}
  \vspace{-13pt}
  \label{fig:method_squence}
\end{figure*}
\begin{definition}[Tube cross-section]
\label{def:tube-section}
Let $\Theta$ be the parameter set. Given a trajectory $p=(p_x,p_u)$ on
$\Tcal=[t_i,t_f]$, a \emph{tube cross-section} is a set-valued map
$\Ecal:\Tcal \rightrightarrows \R^n$ such that for each $t\in\Tcal$ the set
$\Ecal(t)\subset\R^n$ is compact, contains the origin, and is parameterized 
as
\eqn{
\Ecal(t) = \Ecal \big(t;\Theta,\,\overline{w},\,p_x(t),\,p_u(t)\big).
}
By construction, $\Ecal(t)$ bounds the deviation of the true state from $p_x(t)$ for all
$(\theta_f,\theta_g)\in\Theta$ and the bounded disturbance
$\overline w $ satisfying~\cref{asm:noise_derivatives}. The associated tube is
\eqn{
\Omega(t) = p_x(t)\oplus \Ecal(t)
= \{\,x\in\Xcal : x-p_x(t)\in \Ecal(t)\,\}.
}
\end{definition}

\begin{definition}[Robust controlled-invariant (RCI) tube trajectory]
\label{def:robust-tube}
Let $p=(p_x,p_u)$ be a trajectory on $\Tcal=[t_i,t_f]$ as defined in~\eqref{eq:traj_dynamics}.
Let $\Ecal:\Tcal \rightrightarrows \R^n$ be the tube cross-section of~\cref{def:tube-section}, and define
\eqn{
\Omega(t) \;=\; p_x(t)\oplus \Ecal(t).
}
The augmented trajectory is $p^\Omega:\Tcal \to \Xcal\times\Ucal\times 2^{\Xcal}$
\eqn{
p^\Omega(t) = \big(p_x(t),\,p_u(t),\,\Omega(t)\big).
}
We call $p^\Omega$ a \emph{RCI tube trajectory} if:
\begin{enumerate}[label=(\roman*)]
\item The nominal plan satisfies the tightened constraints
\eqn{
p_x(t)\in\overline{\Scal}(t),\quad
p_u(t)\in\overline{\Ucal}(t),\quad
p_x(t_f)\in\overline{\Gcal},
}
where $ \overline{\Scal}(t)=\Scal\ominus \Ecal(t),
\overline{\Gcal}=\Gcal\ominus \Ecal(t_f),
\overline{\Ucal}(t)=\Ucal\ominus \Delta U(t).$
Here $\Delta U:\Tcal\rightrightarrows\R^m$ is any set-valued map with $0\in\Delta U(t)$ for all $t$, representing input deviations relative to $p_u(t)$.\footnote{As an example, see \cite{lopez2019dynamic_robust2} for specific construction of the constraint tightening ($\overline{\Scal}(t)$, $\Delta U(t)$, and $\overline{\Ucal}(t))$.}

\item There exists an ancillary controller
\eqn{
\pi_{\mathrm{arc}}:[t_i,t_f]\times\Xcal\times(\Xcal\times\Ucal\times 2^{\Xcal})\;\to\;\Ucal
}
such that, for all $(\theta_f,\theta_g)\in\Theta$ and disturbances modeled in \cref{asm:noise_derivatives}, the closed loop control law
\eqn{
u(t)=p_u(t)+\pi_{\mathrm{arc}}\big(t,\,x(t),\,p^\Omega(t)\big)
}
renders the tube forward-invariant:
\eqn{
x(t_i)\in \Omega(t_i)\ \Rightarrow\ x(t)\in \Omega(t),\quad \forall\, t\in[t_i,t_f].
}
\end{enumerate}
\end{definition}
A baseline nominal plan for the RCI tube trajectory can be obtained by solving
\begin{subequations}\label{prob:baseline_tube}
\begin{align}
\min_{\,p_x(\cdot),\,p_u(\cdot)}\; & J(p_x,p_u) \\
\text{s.t.}\;\; 
\dot p_x(t) &= f\big(p_x(t),\hat\theta_f\big) + g\big(p_x(t),p_u(t),\hat\theta_g\big), \\
\Omega(t)  &= p_x(t) \oplus \Ecal(t),  \forall t\in[t_0,t_f], \\
p_x(t)     &\in \overline{\mathcal S}(t),  p_u(t)     \in \overline{\mathcal U}(t), \forall t\in[t_0,t_f], \\
p_x(t_f)   &\in \overline{\mathcal G}, x(t_0) \in \Omega(t_0)
\end{align}
\end{subequations}
with mission cost
\begin{equation}\label{eq:mission_cost}
J(p_x,p_u)
= \int_{t_0}^{t_f} \ell\big(p_x(t),p_u(t)\big)\,dt + \ell_T\big(p_x(t_f)\big).
\end{equation}
The baseline ensures safety under worst-case uncertainty, yielding conservative plans. We instead require trajectories to remain safe while actively reducing uncertainty, subject to a prescribed budget on total mission cost.

\begin{prob}
\label{overall_prob}
Consider the system \eqref{eq:system_model}. Design a planning algorithm that, over the mission horizon, generates a sequence of robust tube trajectories whose concatenation defines the planned solution robust trajectory $p^\Omega_{\mathrm{sol}}=(p_x^{\mathrm{sol}},p_u^{\mathrm{sol}},\Omega^{\mathrm{sol}})$ over $[t_0, t_f]$. The solution must satisfy the tightened state and input constraints $p_x(t)\in\overline{\Scal}(t)$, $p_u^{\mathrm{sol}}(t)\in\overline{\Ucal}(t)$ for all $t\in[t_0,t_f]$, reach the goal set $p_x^{\mathrm{sol}}(t_f)\in \overline \Gcal$, and execute exploration to reduce width of the parameter set, i.e., $w_d(\Theta)$ while respecting budget constraint $J\big(p^\Omega_{\mathrm{sol}}\big)\le B$.
\end{prob}

\section{Proposed Solution Framework}
The overall framework is illustrated in Fig.~\ref{fig:method_squence}, which we propose as a solution to~\cref{overall_prob}. The key idea is to actively explore the environment to reduce uncertainty in the unknown parameters, while always guaranteeing safety and respecting the mission cost budget. At a high level, the framework operates as follows: first, a safe backup trajectory is generated (Fig.~\ref{fig:method_squence}[a]); next, multiple \emph{candidate pairs} are constructed, each consisting of a conservative trajectory that guarantees safety and an informative trajectory that seeks to reduce parameter uncertainty (Fig.~\ref{fig:method_squence}[b]). Unsafe candidate pairs are rejected (Fig.~\ref{fig:method_squence}[c]), and among the remaining valid pairs, the most valuable candidate is selected (Fig.~\ref{fig:method_squence}[d]). Finally, once the selected candidate is executed, the backup trajectory is updated iteratively for the next planning cycle (Fig.~\ref{fig:method_squence}[e]). By reducing parameter uncertainty, the robot becomes progressively less conservative in its planning, while ensuring the feasibility of the safety and budget constraints.

At each $t_k$, the algorithm proceeds as follows.

\subsection{Backup trajectory generation}

The process begins by computing the backup trajectory, which serves as a guaranteed safe fallback.

\begin{definition}[Backup trajectory]
\label{def:bak_traj}
The \emph{backup trajectory} generated at time $t_k$, denoted as
$p^{\Omega_{\mathrm{bak}}}_k = (p^{\mathrm{bak}}_{k,x},\,p^{\mathrm{bak}}_{k,u},\,\Omega^{\mathrm{bak}}_k)$, 
is the RCI tube trajectory over $[t_k,t_f]$ that coincides with the backup 
safe set $\Ccal$. In particular,
\eqn{
\Omega^{\mathrm{bak}}_k(t) \;=\; \Ccal_k(t) \;\subseteq\; \Scal(t), 
\quad \forall t \in [t_k,t_f],
}
so that the robust tube trajectory $p_k^{\Omega_{\mathrm{bak}}}$ is equivalent to the 
backup safe set over the planning horizon. Moreover, the ancillary robust 
controller serves as the backup controller $\pi_B$ ensuring that the closed-loop state satisfies
\eqn{
x(t_k)\in\Omega^{\mathrm{bak}}_k(t_k)\ \Longrightarrow\
x(t)\in\Omega^{\mathrm{bak}}_k(t)\subseteq\Scal(t), \forall t \in [t_k,t_f].
}
\end{definition}

Once the backup trajectory is generated, we construct a set of
\emph{candidate trajectories} that the robot may take.

\subsection{Candidate trajectory generation}

At each planning time $t_k$, candidate trajectories are generated in \emph{pairs}: a conservative one that preserves robustness and an informative one that may deviate to reduce uncertainty in $\theta$. The remaining mission is subdivided into horizons of fixed length $T_c$, forming the set
\begin{equation}
\label{eq:set_horizons}
\Tcal^{c}_k =
\big\{\, T^{c,k}_i \;:\; T^{c,k}_i = [\,t_k,\; t_k + iT_c\,]\,\big\}_{i=1}^{N_k},
\end{equation}
where the number of candidate pairs $N_k$ is given by
\begin{equation}
\label{eq:set_horizons_card}
    N_k = \max \big\{\, i \in \mathbb{N} \;:\; t_k + iT_c \leq t_f \,\big\}.
\end{equation}
In other words, starting from $t_k$, candidates are constructed by stacking multiples of $T_c$ until the terminal time would exceed the mission horizon $t_f$.


\begin{definition}[Conservative candidate]
\label{def:commit}
The $i^{th}$ \emph{conservative candidate} generated at time $t_k$ over time horizon $[t_k, t_k + T_i^{c,k}]$ is the robust tube trajectory
\eqn{
p^{\Omega_{\mathrm{cons}, i}}_{k}
  \;=\; \big(p^{\mathrm{cons}, i}_{k,x},\,p^{\mathrm{cons}, i}_{k,u},\,\Omega^{\mathrm{cons}, i}_k\big),
}
defined by
\eqn{
p^{\Omega_{\mathrm{cons}, i}}_{k}(t) \;=\; p^{\Omega_{\mathrm{bak}}}_k(t),
\quad \forall t \in [\,t_k,\,t_k+T_i^{c,k}\,].
}
Thus, the conservative candidate coincides with the backup tube on this
horizon and directly inherits tube invariance and constraint
preservation.
\end{definition}

\begin{definition}[Informative candidate] The $i^{th}$ \emph{informative candidate} generated at time $t_k$ over time horizon $[t_k, t_k + T_i^{c,k}]$ is the trajectory 
\eqn{
p^{\mathrm{info},i}_{k}
  \;=\; \big(p^{\mathrm{info},i}_{k,x},\,p^{\mathrm{info},i}_{k,u}\big)}
designed to reduce the uncertainty of the parameter $\theta$. Unlike the conservative candidate, it may not guarantee safety. However, it must rejoin the conservative at the terminal state:
\eqn{
p^{\mathrm{info},i}_{k,x}\!\big(t_k+T_i^{c,k}\big)
  \;=\; p^{\mathrm{cons},i}_{k,x}\!\big(t_k+T_i^{c,k}\big).
}
\end{definition}
The two trajectories together form a candidate pair.
\begin{definition}[Candidate pair] The $i^{th}$ \emph{candidate pair} generated at time $t_k$ associated with time horizon $[t_k, t_k + T_i^{c,k}]$ is
$\big(p^{\mathrm{info},i}_{k},\,p^{\Omega_{\mathrm{cons}},i}_{k}\big)$.
\end{definition}
Not all candidate pairs are safe, since the informative candidates may violate safety constraints. To capture the subset that is safe, we define:
\begin{definition}[Valid candidate pair]
\label{def:valid_pair}
A candidate pair $\big(p^{\mathrm{info},i}_{k},\,p^{\Omega_{\mathrm{cons}},i}_{k}\big)$
is \emph{valid} if there exists a tube $\Ecal^{\mathrm{info},i}_k(t)$\footnote{$\mathcal E_k^{\text{info},i}$ is a tube cross-section constructed
according to~\cref{def:tube-section}, evaluated along the candidate trajectory
$p_k^{\text{info},i}$.} around
$p^{\mathrm{info},i}_k$ such that
\begin{subequations}
\eqn{
\Omega^{\mathrm{info},i}_k(t) \;&=\; p^{\mathrm{info},i}_{k,x}(t)\,\oplus\,\Ecal^{\mathrm{info},i}_k(t), \\
p^{\Omega_{\mathrm{info}},i}_k 
  \;&=\; \big(p^{\mathrm{info}, i}_{k,x},\,p^{\mathrm{info},i}_{k,u},\,\Omega^{\mathrm{info},i}_{k}\big), 
}    
\end{subequations}
and the tightened constraints hold for all $t\in[\,t_k,\,t_k+T_i^{c,k}\,]$:
\eqn{
\Omega^{\mathrm{info},i}_k(t)\subseteq \overline\Scal(t), \quad
p^{\mathrm{info},i}_{k,u}(t)\in \overline\Ucal(t).
}
Both $p^{\Omega_{\mathrm{info}},i}_k$ and $p^{\Omega_{\mathrm{cons}},i}_k$ are RCI tube trajectories safe over the horizon, and the pair is \emph{valid}.
\end{definition}
The collection of all valid pairs generated at time $t_k$ is defined as the set,
\begin{equation}
\Vcal^{c}_k
=
\Big\{
\big(p^{\mathrm{info},i}_{k},\,p^{\Omega_{\mathrm{cons}},i}_{k}\big)
:\;
 \text{pair valid by \textit{Def}.~\ref{def:valid_pair}}\Big\}.
\end{equation}
We also define the set of all conservative candidates as 
\begin{equation}
\Vcal_{k}^{\mathrm{cons}} =
\big\{\, p^{\Omega_{\mathrm{cons}},i}_{k} \big\}_{i = 1}^{N_k}.
\end{equation}

\subsection{Committing the candidate trajectory}
Having defined the sets of valid pairs 
$\Vcal^{c}_k$ and conservative candidates 
$\Vcal_{k}^{\mathrm{cons}}$, the final step is to decide which trajectory to commit at time $t_k$. 
\textit{The guiding principle is to drive uncertainty reduction as quickly as possible, while never compromising safety or exceeding the budget.}
To implement this, each valid pair is scored according to the predicted reduction in uncertainty achieved by its informative trajectory, discounted by horizon length. 
The highest-scoring pair that satisfies the budget is selected; if none exist, the choice defaults to the conservative candidate with the smallest horizon, which is always safe. 
In this way, the executed trajectory ensures safety while consistently prioritizing rapid uncertainty reduction. 
We formalize this process as follows.

Each valid candidate pair 
$\big(p^{\Omega_{\mathrm{info},i}}_{k},\,p^{\Omega_{\mathrm{cons}},i}_{k}\big) \in \Vcal^{c}_k$
is assigned a score reflecting predicted uncertainty reduction along its informative trajectory. To prioritize early information gain, the score is penalized by horizon length:
\begin{equation}
    s^{c,k}_i \;=\; \exp\!\big(-\lambda T^{c,k}_i\big)\,\Delta w_i, 
    \qquad \lambda > 0,
\label{eq:score}
\end{equation}
where $\Delta w_i$ is the reduction in the \emph{average} directional width. If $\Theta^k$ is the current uncertainty set and $\Theta^{k+1,i}$ the predicted set after executing candidate $i$, then
\begin{equation}
    \Delta w_i \;=\; \frac{1}{|\Dcal|}\,\sum_{d \in \Dcal} \Big( w_d(\Theta^k) - w_d(\Theta^{k+1,i}) \Big),
\label{eq:avg-width-reduction}
\end{equation}
with $w_d(\Theta)$ as in \textit{Def}.~\ref{def:width}. The computation of $w_d(\Theta)$ is given in \cref{sec:uncertainty_pred}. 

To enforce budget feasibility, let $J^k_{\mathrm{exec}}$ be the cost of the committed trajectory up to $t_k$, and $J^k_{\mathrm{back}}$ the cost of the conservative backup from $t_k$ to $t_f$. For a candidate horizon $T^{c,k}_i$, the excess cost of the informative trajectory over its conservative counterpart is
\begin{equation}
    \Delta J_i \;=\; J\big(p^{\Omega_{\mathrm{info},i}}_{k}\big)
    \;-\; J\big(p^{\Omega_{\mathrm{cons}},i}_{k}\big).
\label{eq:cost-diff}
\end{equation}
A candidate pair is feasible if, after executing its informative segment and reverting to the backup, the total mission cost does not exceed the budget $B$. Equivalently,
\begin{equation}
    \Fcal^{c}_k =
    \Big\{\, i : 
    \begin{aligned}[t]
        \big(p^{\Omega_{\mathrm{info},i}}_{k},&\,p^{\Omega_{\mathrm{cons}},i}_{k}\big) \in \Vcal^{c,k}, \\
        &J^k_{\mathrm{exec}} + J^k_{\mathrm{back}} + \Delta J_i \;\leq B
    \end{aligned}
    \,\Big\}.
\label{eq:feasible-set}
\end{equation}
Only those valid pairs that satisfy the budget constraint are retained in $\Fcal^{c}_k$. If $\Fcal^{c}_k  \neq \emptyset$, we select the candidate with the highest score:
\begin{equation}
    i^\star \;=\; \arg\max_{i \in \Fcal^{c,k}} \; s^{c,k}_i.
\label{eq:best-candidate}
\end{equation}
The informative trajectory of the selected pair is then committed. If $\Fcal^{c,k}=\emptyset$, no valid pair or all over budget, we fall back to the shortest horizon conservative candidate. 

\begin{definition}[Committed trajectory]
At planning time $t_k$, the committed trajectory is defined as
\begin{equation}
p^{\Omega_{\mathrm{com}}}_k(t) \;=\;
\begin{cases}
    p^{\Omega_{\mathrm{info},i^\star}}_{k}(t), & \Fcal^{c,k} \neq \emptyset, \\[6pt]
    p^{\Omega_{\mathrm{cons},1}}_{k}(t), & \Fcal^{c,k} = \emptyset,
\end{cases}
\quad t \in [t_k,t_f],
\label{eq:committed}
\end{equation}
where $p^{\Omega_{\mathrm{cons},1}}_{k}$ denotes the conservative candidate corresponding to the smallest horizon in the set $\Vcal_{k}^{\mathrm{cons}}$.
\end{definition}
After executing the committed trajectory, we measure shrinkage via set membership identification (SMID) \cite{milanese2004set, lopez2019adaptive}, update the backup, and then the process repeats. Candidate generation, prediction of $w_d(\Theta)$, and theoretical guarantees appear in Sections~\ref{sec:cand_gen}, \ref{sec:uncertainty_pred}, and \ref{sec:theo_gua}.

\section{Constructing Informative Candidate Trajectories}
\label{sec:cand_gen}

At each planning time $t_k$, we generate candidate trajectories that may deviate from the conservative backup to promote identification. The idea is to excite the regressor $\Phi(x,u)$ and accelerate the reduction of uncertainty in the admissible set $\Theta_k$. This section details the construction of one informative candidate trajectory.

Let $T_i^{c,k}\in\Tcal^{c}_k$ be any horizon from the set defined in equations \eqref{eq:set_horizons} and \eqref{eq:set_horizons_card}. For this horizon, we compute the informative candidate
$p^{\text{info},i}_k = \big(p^{\text{info},i}_{k,x},\,p^{\text{info},i}_{k,u}\big)$
by solving the following optimization over $[t_k,\,t_k+T_i^{c,k}]$:
\begin{subequations}
\label{opt_cand_prob}
 \eqn{
 \min_{x(\cdot),\,u(\cdot)} \quad &
\int_{t_k}^{t_k+T_i^{c,k}} \ell\big(x(\tau),u(\tau)\big)\,d\tau - \gamma\,\log\det\big(\Gcal_k(T_i^{c,k})\big) \\
\text{s.t.}\quad &
\dot{x}(t)= f\big(x(t),\hat\theta_f\big)+g\big(x(t),\hat\theta_g\big)u(t), \\
& x(t_k)=x_k,\\
& x(t_k+T_i^{c,k}) = p^{\Omega,\mathrm{cons},i}_{k,x}(t_k+T_i^{c,k}), \label{rejoin_cons}
 }   
\end{subequations}
Here $\rho,\gamma>0$ weight control effort and identifiability; $(\hat\theta_f,\hat\theta_g)$ are nominal parameters; and $\Gcal_k(T)$ is the finite–horizon Gramian (\textit{Def}.~\ref{def:PE}). The $\log\det(\Gcal_k(\cdot))$ term promotes regressor excitation, \eqref{rejoin_cons} reconnects to the backup. See \cite{PE_1_narendra1987persistent} on choosing $\alpha$. Let $(x^\star(\cdot),u^\star(\cdot))$ denote the solution of \eqref{opt_cand_prob}. We set
$p^{\text{info},i}_{x,k}(t)=x^\star(t),\quad
p^{\text{info},i}_{u,k}(t)=u^\star(t),\quad \forall t\in[t_k,t_k+T_i^{c,k}].$

\begin{remark}During informative candidate generation, robustness to bounded uncertainties and satisfaction of the safety constraints are not explicitly enforced. After obtaining $p^{\mathrm{info},i}_k$, safety is assessed by constructing a tube $\Omega^{\mathrm{info},i}_k$ around the trajectory. If the induced tube satisfies the tightened constraints, the valid pair $\big(p^{\mathrm{info},i}_k,\, p^{\Omega,\mathrm{cons},i}_k\big)$ is formed according to \textit{Def}.~\ref{def:valid_pair}.
\end{remark}

\section{Uncertainty shrinkage prediction}
\label{sec:uncertainty_pred}

In this section, we quantify the pre-execution impact of a planned trajectory on parameter uncertainty via the \emph{width} in \textit{Def}.~\ref{def:width}. We present the method for a general trajectory $p=(p_x,p_u)$ on $[t_i,t_f]$. 
Let $t_j$, $j=1,\dots,N_j$ denote the sampling times, and define $
\Phi_j=\Phi\big(p_x(t_j),\,p_u(t_j)\big)\in\R^{c\times p},\quad
z_j=z(t_j)\in\R^c,\quad
w_j=w(t_j)\in\R^c,\;\; \|w_j\|_\infty\le \overline w .
$
Stacking the regressors gives
\eqn{
A =
\begin{bmatrix}\Phi_1 & \Phi_2 & \cdots & \Phi_{N_j}\end{bmatrix}^{\top}
\in\R^{M\times p},\quad M=N_j c .
\label{eq:Astack}
}

Now let $\theta^\star$ denote the true (unknown) parameter. Two parameters $\theta$ and $\theta^\star$ can produce the same stacked data under bounded noises $w_1,w_2$ if and only if
\eqn{
A\theta + w_1 = A\theta^\star + w_2, 
\quad 
\|w_1\|_{\infty},\|w_2\|_{\infty} \le \overline w .
}
The key question is: \textit{under what conditions can two distinct parameters $\theta$ and $\theta^\star$ produce the same measurements within the noise bound?}
If many such parameters remain feasible, the uncertainty set stays large; if few remain, it shrinks.
Hence, the fewer alternative parameters that fit the data within the noise bound, the greater the uncertainty reduction achieved by the planned trajectory.
\begin{remark}
The uncertainty shrinkage prediction in this section characterizes uncertainty reduction along a planned trajectory under bounded additive disturbances and bounded parametric uncertainty, but does not account for deviation between the planned trajectory and the executed closed-loop state-input trajectory, which is left for future work.
\end{remark}
\begin{lemma}
\label{lemma:2w}
Let $\theta^\star\in\Theta$ be the (unknown) true parameter, and define 
$e_\theta=\theta-\theta^\star\in\R^p$. 
There exists some $w_j$ with $\|w_j\|_{\infty}\le \overline w$ such that
\eqn{
\|\Phi_j(\theta^\star-\theta)+w_j\|_\infty 
\le \overline w 
\iff 
\|\Phi_j e_\theta\|_\infty \le 2\overline w ,
\label{eq:single-iff}
}
and combining across all $j$ yields $\|A e_\theta\|_\infty \le 2\overline w$.
\end{lemma}
\begin{proof} 

We first show that \eqref{eq:single-iff} holds. 
($\Rightarrow$) Since $e_{\theta} = \theta - \theta^\star$, we have $\Phi_j(\theta^\star-\theta) = -\Phi_j e_{\theta}$. Thus \begin{subequations}\label{lemma:trick1}
\eqn{
\|\Phi_j(\theta^\star-\theta)+ w_j\|_\infty
&= \|-\Phi_j e_\theta+ w_j\|_\infty \\
&= \|\Phi_j e_\theta- w_j\|_\infty \le \overline w 
}    
\end{subequations}
Now, by the triangle inequality:
\begin{subequations}
 \eqn{
\|\Phi_j e_\theta\|_\infty 
&= \|(\Phi_j e_\theta -w_j) + w_j \|_\infty \\ 
&\le \|\Phi_j e_\theta- w_j\|_\infty + \|w_j\|_\infty \\
&\le \overline w+\overline w
=2\overline w .
}   
\end{subequations}

($\Leftarrow$) Let  $\|\Phi_j e_\theta\|_\infty\le 2\overline w$, choose 
$w_j =\,\mathrm{clip}(\Phi_j e_\theta,\,\overline w)$, where $\mathrm{clip}(\Phi_j e_\theta,\,\overline w) = \mathrm{sign}(\Phi_j e_\theta) \odot \min\{|\Phi_j e_\theta|,\overline w\}$. Thus, 
\begin{subequations}
    \eqn{
    \|w_j\|_\infty &\le \overline w \\
    \|\Phi_j e_\theta - w_j\|_\infty &= \max \{\|\Phi_j e_\theta\|_\infty - \overline w, 0\} \le \overline w 
    }
\end{subequations}
Since $\Phi_j(\theta^\star-\theta) = -\Phi_j e_{\theta}$ and using \eqref{lemma:trick1}, we have 
\eqn{
\|\Phi_j(\theta^\star-\theta)+ w_j\|_\infty = \|\Phi_j e_\theta- w_j\|_\infty 
\le \overline w
}
Thus, combining \eqref{eq:single-iff} for all $j$ yields $\|A e_\theta\|_\infty \le 2\overline{w}$.
\end{proof}

\cref{lemma:2w} states that two parameters are \emph{indistinguishable} 
when their offset $e_\theta$ satisfies $\|A e_\theta\|_\infty \le 2\overline{w}$, 
meaning that the observed data could equally well be explained by either parameter, given bounded noise. Offsets exceeding $2\overline{w}$ are ruled 
out, motivating the definition of the feasible error set.
\begin{definition}[Error set]\label{def:Ecal}
Given the planned trajectory (through $A$), $\overline w$, and $e_\theta = \theta - \theta^\star$, define
\eqn{
\Ecal_\theta \;:=\; \{\,e_\theta\in\R^p:\ \|A e_\theta\|_\infty \le 2\overline w\,\}.
\label{eq:E-def}
}
\end{definition}
The error set $\Ecal_\theta := \{e_\theta\in\R^p : \|Ae_\theta\|_\infty \le 2\overline w\}$ collects all feasible offsets $e_\theta = \theta-\theta^\star$ under the trajectory and noise bound. 
Equivalently, $\theta$ is feasible iff $\theta-\theta^\star \in \Ecal_\theta$, i.e., $\theta \in \theta^\star + \Ecal_\theta$. 
Since $\theta \in \Theta$, the feasible parameter set is
\eqn{
\Theta_{N_j}(\theta^\star) \;=\; \Theta \cap (\theta^\star + \Ecal_\theta).
}
where $\Theta_{N_j}(\theta^\star)$ is the predicted parameter set after $N_j$ planned samples. We measure the width of the predicted set $\Theta_{N_j}(\theta^\star)$ along a direction $d\in\R^p$. 
Choosing $d=e_{\theta,i}$ yields the width of the $i$-th parameter. 
By definition,
\eqn{
w_d \big(\Theta_{N_j}(\theta^\star)\big) 
= w_d \big(\Theta \cap (\theta^\star+\Ecal_\theta)\big).
}
Because the width of an intersection cannot exceed that of either set, and width is translation-invariant ($w_d(\theta^\star+\Ecal_\theta)=w_d(\Ecal_\theta)$), we obtain
\eqn{
w_d \big(\Theta_{N_j}(\theta^\star)\big) 
\;\le\; \min \!\big(w_d(\Theta),\, w_d(\Ecal_\theta)\big).
}
We next show how $w_d(\Ecal_\theta)$ can be computed from the planned trajectory.

\begin{figure*}[t]
  \centering
  \includegraphics[width=2.05\columnwidth]{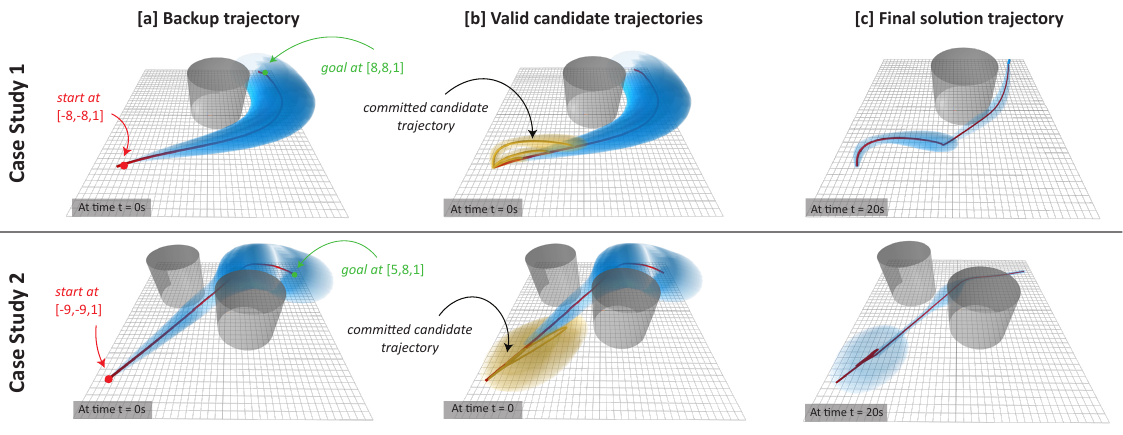}
  \caption{Backup, candidate, and final solution trajectories for Case Study~1 (top) and Case Study~2 (bottom).}
  \vspace{-13pt}
  \label{fig:trajectories}
\end{figure*}
\begin{lemma}\label{lem:width-2h}
 For any $d\in\R^p\!\setminus\!\{0\}$,
\eqn{
\label{eq:width-is-2h}
w_d(\Ecal_\theta)=2\,h_{\Ecal_\theta}(d),
\qquad
h_{\Ecal_\theta}(d)=\sup_{e_\theta \in \Ecal_{\theta}} d^\top e_\theta}
\end{lemma}
\begin{proof}

For any $e_\theta \in \mathcal{E}_\theta$,
\eqn{
\|Ae_\theta\|_\infty = \|-Ae_\theta\|_\infty =
\|A(-e_\theta)\|_\infty   \le 2\overline w.
}
Since $-e_\theta \in \mathcal{E}_\theta$ and $0 \in \mathcal{E}_\theta$, this implies $w_d(\Ecal_\theta) = \sup_{e_\theta \in \Ecal_\Theta} d^\top e_\theta - \inf_{e_\theta \in \Ecal_\Theta} d^\top e_\theta = 2\sup_{e_\theta \in \Ecal_\Theta} d^\top e_\theta.$
\end{proof}

Since $\|Ae_\theta\|_\infty \leq 2\overline w \iff -2\overline w\mathbf{1}_M \le A e_\theta \le 2\overline w\mathbf{1}_M$,
by \eqref{eq:width-is-2h} it suffices to compute
\eqn{
\label{eq:lp_eq}
\begin{aligned}
\max_{e_\theta\in\R^p}\;& d^\top e_\theta \\
\text{s.t.}\;& -2\overline w\mathbf{1}_M \le A e_\theta \le 2\overline w\mathbf{1}_M .
\end{aligned}
}
Now, we introduce multipliers $\lambda_1,\lambda_2\in\R^M_{\ge 0}$ and write the Lagrangian of the \eqref{eq:lp_eq} as $\Lcal(e_\theta,\lambda_1,\lambda_2)
= d^\top e_\theta + \lambda_1^\top(2\overline w\mathbf 1_M - Ae_\theta) + \lambda_2^\top(2\overline w\mathbf 1_M + Ae_\theta).$
The dual is finite iff $A^\top(\lambda_2-\lambda_1)=d$, giving
\eqn{
\begin{aligned}
h_{\Ecal_\theta}(d) 
&= \min_{\lambda_1,\lambda_2\ge 0} 
\;2\overline w\big(\mathbf 1_M^\top\lambda_1+\mathbf 1_M^\top\lambda_2\big) \\
&\text{s.t.}\quad A^\top(\lambda_2-\lambda_1)=d .
\end{aligned}
}
Let $\lambda_d=\lambda_2-\lambda_1$. Choosing 
$\lambda_2=(\lambda_d)_+$ and $\lambda_1=(\lambda_d)_-$ yields
\eqn{
\label{eq:prop_des_form}
h_{\Ecal_\theta}(d)=2\overline w \min_{\lambda_d:\,A^\top\lambda_d=d}\|\lambda_d\|_1 .
}
Therefore,
\eqn{
w_d\big(\Theta_{N_j}(\theta^\star)\big) 
\le \min\big(w_d(\Theta),\; 2h_{\Ecal_\theta}(d)\big).
}

If the executed closed-loop state-input trajectory coincides with the planned trajectory, then the predicted width upper bounds the width computed from the realized measurements, since the prediction is computed under worst-case bounded additive disturbances.

\section{Theoretical guarantees}
\label{sec:theo_gua}
This section establishes the safety and budget feasibility of the proposed framework, showing that the solution defined by the committed tube sequence satisfies the safety constraints and the prescribed budget (\cref{thm:psol_safe_budget}).

\begin{assumption}
\label{assump:mono-tube-cost}
Over time horizon $[t_k,t_f]$, let $p^{\Omega(\Theta)}$ be the RCI tube trajectory returned by \eqref{prob:baseline_tube} when the tube cross-section $\Ecal(t;\Theta)$ is built from $\Theta$ (\textit{Defs}.~\ref{def:tube-section},\ref{def:robust-tube}). 
We assume monotonicity under set inclusion:
\eqn{
\Theta' \subseteq \Theta \ \Longrightarrow\ J\big(p^{\Omega(\Theta')}\big) \le J\big(p^{\Omega(\Theta)}\big).
}
\end{assumption}

\begin{theorem}
\label{thm:psol_safe_budget}
Let $t_0<t_1<\cdots<t_k\cdots$ be the commit times and let $p_k^{\Omega\mathrm{com}}$ denote the tube committed at $t_k$ (Def.~\ref{def:commit}). 

Define the solution as the set of committed tubes
\eqn{
p^\Omega_{\mathrm{sol}}=\{\,p_0^{\Omega\mathrm{com}},\,p_1^{\Omega\mathrm{com}},\,\cdots,\,p_k^{\Omega\mathrm{com}}\, \cdots\}.
}
Assume $J_{\mathrm{back}}^0\le B$. At each $t_k$, generate valid pairs (Def.~\ref{def:valid_pair}), form the budget–feasible set $F_j^c$ given in~\eqref{eq:feasible-set}, and commit per Def.~\ref{def:commit}. 
If the system tracks each committed tube, then $p^\Omega_{\mathrm{sol}}(t)\in \overline{S}(t)$ for all $t\in[t_0,t_f]$ and $J\big(p^\Omega_{\mathrm{sol}}\big)\le B$.
\end{theorem}

\begin{proof}
The argument follows by induction, adapting the reasoning in \cite{agrawal2024gatekeeper}. 
At each decision time $t_k$, the committed trajectory is either (i) the conservative RCI tube, or (ii) the informative element of a valid pair. 
In the latter case, the informative trajectory is tube-safe on $[t_k,t_{k+1}]$ and is guaranteed to rejoin a robust tube at $t_{k+1}$. 
By construction, concatenation of tube-safe trajectories preserves safety over time. Budget feasibility is verified at time $t_k$ against the backup defined on $[t_k,t_f]$ by defining the feasible set~\eqref{eq:feasible-set}.
By \cref{assump:mono-tube-cost}, the backup recomputed at $t_{k+1}$ is no more expensive than the backup used in the feasibility check at $t_k$. 
Therefore, the cost of the planned trajectories, given by the sum of the committed informative segment and the subsequently applied backup, does not exceed the prescribed budget. 
If no informative candidate satisfies the feasibility condition, the conservative tube is selected, in which case the accumulated cost does not increase. Beyond the last decision time, the system evolves along the terminal tube until $t_f$, yielding a solution that satisfies
\eqn{
p^\Omega_{\mathrm{sol}}(t) \in \mathcal \overline{S}(t), 
\quad \forall t \in [t_0,t_f],
\quad
J\left(p^\Omega_{\mathrm{sol}}\right) \le B .
}
\end{proof}

\section{Results \& Discussion}\label{sec:results}

\begin{table}[ht]
\centering
\scriptsize
\begin{tabular}{c|c|c|c|c|c}
\hline
\multirow{2}{*}{\textbf{System}} &
\multirow{2}{*}{\textbf{Method}} &
\multirow{2}{*}{\makecell{\textbf{Budget}\\\textbf{(\%)}}} &
\multirow{2}{*}{\makecell{\textbf{Total Cost}\\\textbf{(\%)}}} &
\multicolumn{2}{c}{\textbf{Uncertainty Reduction (\%)}} \\
\cline{5-6}
 &  &  &  & \textbf{Param 1} & \textbf{Param 2} \\
\hline\hline
\multirow{2}{*}{1}
& Baseline 1 & --    & 100.0 & 0  & -- \\
& Proposed   & 110.0 & 82.5  & 88.0 & -- \\
\hline
\multirow{2}{*}{2}
& Baseline 1 & --    & 100.0 & 0  & 0 \\
& Proposed   & 110.0 & 81.3  & 34.0 & 88.8 \\
\hline
\end{tabular}
\caption{Budget and cost are shown as percentages relative to the baseline solution (100\%). 
For System 2, reductions are reported per parameter dimension.}
\label{tab:comparison}
\end{table}

We validate the framework on two quadrotor systems. 
Backup trajectories are generated with tube MPC \cite{lopez2019dynamic_robust2} and tracked using a sliding mode controller. 
The cost functional,
\eqn{
    J(p_x,p_u) \;=\; \int_{t_0}^{t_f} \big( \alpha \|u(t)\|^2 + \beta \|p_x(t)-r_{\text{goal}}\|^2 \big)\,dt,
}
penalizes control effort and deviation from the goal, where $\alpha,\beta>0$ are weights, $p_x(t)$ the nominal state, $u(t)$ the input, and $r_{\text{goal}}$ the goal. 
In both cases we set $T_C=2.0$s.

\subsubsection{Case Study 1: Quadrotor with Drag}
The first case study, illustrated in Fig.~\ref{fig:trajectories}, considers a quadrotor modeled as a double integrator with nonlinear aerodynamic drag,
\eqn{
    \ddot{r} = -C_d \|\dot{r}\| \dot{r} + g + u + d,
    \label{eq:drag_dynamics}
}
where $r \in \R^3$ is the inertial position, $g \in \R^3$ the gravitational acceleration, $u \in \R^3$ the control input, $d \in \R^3$ an additive disturbance, and $C_d \in \R$ the unknown drag coefficient. The measurement $y \in \R^3$ corresponds to $r$.  

As shown in Fig.~\ref{fig:parameter_bounds}, committing a $6$-second informative trajectory reduced the admissible interval for $C_d$, tightening the bounds around the true parameter and demonstrating the framework’s ability to shrink parametric uncertainty online.

\begin{figure*}[t]
  \centering
  \includegraphics[width=2.05\columnwidth]{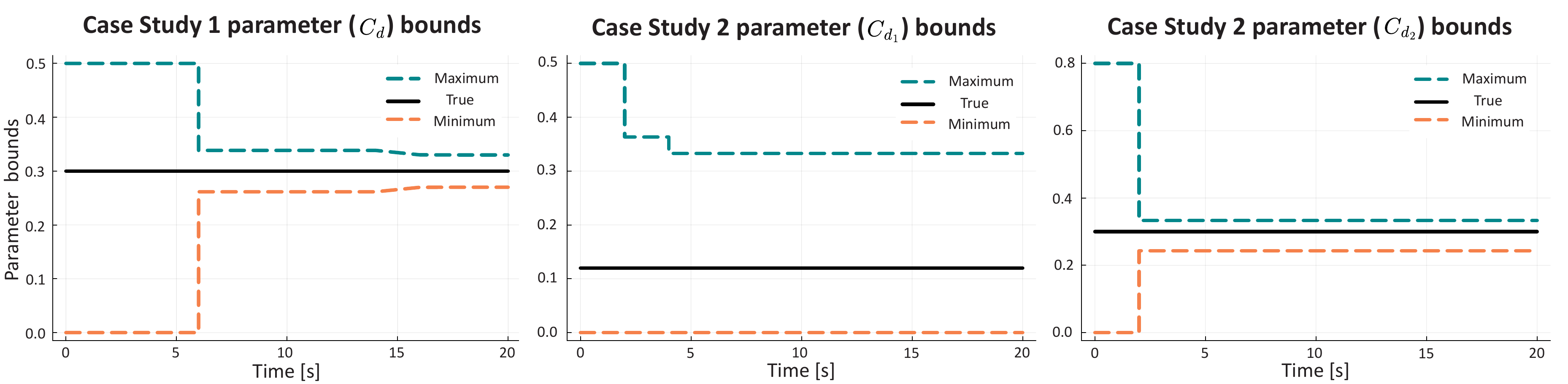}
  \caption{Parameter bound evolution are shown for two studies: Case Study~1 ($C_d$, left) and Case Study~2 ($C_{d_1}$, middle; $C_{d_2}$, right).}
  \vspace{-13pt}
  \label{fig:parameter_bounds}
\end{figure*}
Table~\ref{tab:comparison} reports the corresponding mission cost. Relative to the conservative baseline (set to $100\%$), the proposed method achieved only $82.5\%$ of the cost, while remaining within the budget of $110\%$. Thus, the approach not only reduced parameter uncertainty but also improved overall efficiency compared to the baseline backup solution.

\subsubsection{Case Study 2: Quadrotor with Vector Drag}
We extend the setup of Case Study~1 by considering a quadrotor with vector drag dynamics:
\begin{equation}\label{eq:quad_vec_drag}
    \ddot r = -C_{d_1}\dot r - C_{d_2}\,\|\dot r\|\,\dot r + g + u + d,
\end{equation}
where $C_{d1},C_{d2}\in\R$ are the unknown drag coefficients.  

In this case, the robot executed informative trajectories that reduced the parameter set in both directions, tightening the bounds from $[0.0,\,0.50]\times[0.0,\,0.80]$ to $[0.0,\,0.33]\times[0.25,\,0.34]$. 
The asymmetric shrinkage reflects the relative excitation of the regressors: more data were collected along $C_{d2}$, yielding a stronger contraction of its bounds. 
Since candidate generation did not enforce safety, and the robot operated in a narrow corridor, many informative candidates were invalidated. 
Consequently, fewer safe informative trajectories were committed, producing only modest reduction in $C_{d1}$ compared to $C_{d2}$. 
The total cost was $81.3\%$, well below the $110\%$ budget, normalized to the $100\%$ baseline.

\subsubsection{Implementation details}

In practice, generating a robust tube trajectory based on the method described in \cite{lopez2019dynamic_robust2} requires solving a nonlinear optimization problem and can take anywhere from 100 milliseconds to one second, depending on the problem size and solver warm-start conditions. The generation of an informative candidate trajectory is typically faster, on the order of a few hundred milliseconds, though this is also highly dependent on the system dynamics and horizon length. Both the backup and candidate trajectory optimization problems are implemented using InfiniteOpt.jl, which provides a flexible framework for modeling dynamic optimization problems in Julia.

\section{Conclusion}
We present a dual control framework that guarantees safety and budget feasibility while reducing parameter uncertainty.
By committing only trajectories that satisfy safety and cost constraints, the framework enables reliable execution with uncertainty reduction.
Case studies show lower mission cost than a conservative baseline and tighter parameter bounds.
Future work will incorporate safety in candidate generation, quantify the link between uncertainty reduction and cost decrease to formalize dual control trade-offs, and develop prediction methods for the width of the parameter set that account for tracking error rather than neglecting deviations between planned and executed trajectories.

\nocite{*}
\bibliographystyle{IEEEtran}
\bibliography{main.bib}

\end{document}

%% file: preamble.tex
\usepackage{graphics}
\usepackage{epsfig}

\usepackage{cite}
\usepackage[dvipsnames, svgnames]{xcolor}
\usepackage{xurl}
\usepackage{hyperref}
\hypersetup{
    colorlinks,
    linkcolor={red!50!black},
    citecolor={blue!50!black},
    urlcolor={blue!80!black}
}
\usepackage{bm}
\let\labelindent\relax
\usepackage{enumitem}
\usepackage{breqn}
\usepackage{amssymb}
\usepackage{tabularx}
\usepackage{booktabs}
\usepackage{dcolumn}
\newcolumntype{d}[1]{D..{#1}}
\newcolumntype{C}{>{\centering\arraybackslash}X}
\usepackage{soul}
\usepackage[linesnumbered,ruled,vlined]{algorithm2e}

\usepackage{amsthm}

\newcommand{\reals}{\mathbb{R}}

\newcommand{\R}{\reals}

\newcommand{\Ccal}{\mathcal{C}}
\newcommand{\Dcal}{\mathcal{D}}
\newcommand{\Ecal}{\mathcal{E}}
\newcommand{\Fcal}{\mathcal{F}}
\newcommand{\Gcal}{\mathcal{G}}

\newcommand{\Lcal}{\mathcal{L}}

\newcommand{\Scal}{\mathcal{S}}
\newcommand{\Tcal}{\mathcal{T}}
\newcommand{\Ucal}{\mathcal{U}}
\newcommand{\Vcal}{\mathcal{V}}
\newcommand{\Xcal}{\mathcal{X}}

\newcommand{\eqn}[1]{\begin{align} #1 \end{align}}

\theoremstyle{plain}
\newtheorem{theorem}{Theorem}

\newtheorem{lemma}{Lemma}
\theoremstyle{definition}
\newtheorem{definition}{Definition}

\theoremstyle{remark}

\definecolor{yellow}{cmyk}{0.0,0.10,0.95,0.0}
\definecolor{pred}{cmyk}{0,0.8,0.70,0.0}
\definecolor{bluedefined}{cmyk}{0.46, 0.10, 0, 0.0}

\usepackage{makecell}
\usepackage[table]{xcolor}

%% file: subsections/abstract.tex
Planning safe trajectories under model uncertainty is a fundamental challenge. Robust planning ensures safety by considering worst-case realizations, yet ignores uncertainty reduction and leads to overly conservative behavior. Actively reducing uncertainty on-the-fly during a nominal mission defines the dual control problem. Most approaches address this by adding a weighted exploration term to the cost, tuned to trade off the nominal objective and uncertainty reduction, but without formal consideration of when exploration is beneficial. Moreover, safety is enforced in some methods but not in others. We propose a framework that integrates robust planning with active exploration under formal guarantees as follows: The key innovation and contribution is that exploration is pursued only when it provides a verifiable improvement without compromising safety. To achieve this, we utilize our earlier work on \gatekeeper{} as an architecture for safety verification, and extend it so that it generates both safe and informative trajectories that reduce uncertainty and the cost of the mission, or keep it within a user-defined budget. The methodology is evaluated via simulation case studies on the online dual control of a quadrotor under parametric uncertainty.  \href{https://kalebbennaveed.github.io/tro_dual_gtk/}{[Paper Website]}\footnote{https://kalebbennaveed.github.io/tro_dual_gtk/}